\documentclass[11pt]{article}
\usepackage{microtype}
\usepackage{graphicx}
\usepackage[margin=1.15in]{geometry}
\usepackage{booktabs} 

\usepackage[utf8]{inputenc} 

\usepackage{hyperref}       
\usepackage{url}            

\usepackage{amsfonts}       
\usepackage{nicefrac}
\usepackage{amsthm}
\usepackage{amsmath}
\allowdisplaybreaks

\usepackage{color}
\usepackage{pifont}
\usepackage{xspace}
\usepackage{bbold}
\usepackage{amssymb}
\usepackage{enumerate}
\usepackage{subcaption}
\usepackage{wrapfig}
\usepackage{tikz}
\usetikzlibrary{positioning}
\usetikzlibrary{shadows}
\usetikzlibrary{arrows}
\usetikzlibrary{shapes}
\usepackage{pgfplots}
\pgfplotsset{compat=1.9}
\newtheorem{thm}{Theorem}
\newtheorem{lemma}[thm]{Lemma}
\newtheorem{corollary}[thm]{Corollary}
\newtheorem{assump}{Assumption}

\newtheorem{definition}{Definition}

\newcommand{\xs}{x^*}
\newcommand{\E}{\mathbb{E}}
\newcommand{\EE}{\mathfrak{E}}
\newcommand{\one}{\mathbb{1}}

\newcommand{\F}{\mathcal{F}}

\newcommand{\R}{\mathbb{R}}
\newcommand{\D}{\mathbb{D}}

\newtheoremstyle{TheoremNum}
{\topsep}{\topsep}              
{\itshape}                      
{}                              
{\bfseries}                     
{.}                             
{ }                             
{\thmname{#1}\thmnote{ \bfseries #3}
	{\normalfont(Main Theorem)}}
\theoremstyle{TheoremNum}
\newtheorem{thmn}{Theorem}

\newtheoremstyle{TheoremNumn}
{\topsep}{\topsep}              
{\itshape}                      
{}                              
{\bfseries}                     
{.}                             
{ }                             
{\thmname{#1}\thmnote{ \bfseries #3}
	{\normalfont(Restated)}}
\theoremstyle{TheoremNumn}

\begin{document}
\title{An Alternative View: When Does SGD Escape Local Minima?}

\author{
	Robert Kleinberg \\
	Department of Computer Science\\
	Cornell University \\
	\texttt{rdk@cs.cornell.edu}\\
	\and
	Yuanzhi Li \\
	Department of Computer Science\\
	Princeton University \\
	\texttt{yuanzhil@cs.princeton.edu}
	\and
	Yang Yuan \\
	Department of Computer Science\\
	Cornell University \\
	\texttt{yangyuan@cs.cornell.edu}
}
\maketitle

\begin{abstract}
Stochastic gradient descent (SGD) is widely used in machine learning. Although being commonly viewed as a fast but not accurate version of gradient descent (GD), it always finds better solutions than GD for modern neural networks. 

In order to understand this phenomenon, we take an alternative view that SGD is working on the convolved (thus smoothed) version of the loss function. We show that, even if the function $f$ has many bad local minima or saddle points, as long as for every point $x$, the weighted average of the gradients of its neighborhoods is one point convex with respect to the desired solution $x^*$, SGD will get close to, and then stay around $x^*$ with constant probability. More specifically, SGD will not get stuck at ``sharp'' local minima with small diameters, as long as the neighborhoods of these regions contain enough gradient information. The neighborhood size is controlled by step size and gradient noise. 

Our result identifies a set of functions that SGD provably works, which is much larger than the set of convex functions. Empirically, we observe that the loss surface of neural networks enjoys nice one point convexity properties locally, therefore our theorem helps explain why SGD works so well for neural networks.
\end{abstract}


\section{Introduction}
Nowadays, stochastic gradient descent (SGD), as well as its variants (Adam \cite{adam}, Momentum \cite{momentum}, Adagrad \cite{adagrad}, etc.) have become the de facto algorithms for training neural networks.
SGD runs iterative updates for the weights $x_t$:
$x_{t+1}=
x_t-\eta v_t$, where $\eta$ is the step size\footnote{In this paper, we use step size and learning rate interchangeably.}. $v_t$ is the stochastic gradient that satisfies $E[v_t]=\nabla f(x_t)$, and is usually computed using a mini-batch of the dataset. 

In the regime of convex optimization, SGD is proved to be a nice tradeoff between accuracy and efficiency: it requires more iterations to converge, but fewer gradient evaluations per iteration. Therefore, for the standard empirical risk minimizing problems with $n$ points and smoothness $L$,
to get to $\epsilon$-close to $x^*$,
GD needs $O(Ln/\epsilon)$ gradient evaluations \cite{nesterov2013introductory}, but SGD with reduced variance
only needs $O(n \log \frac1{\epsilon}+\frac{L}\epsilon)$ gradient evaluations \cite{SVRG,SAGA,SAG,SVRG++}. 
In these scenarios, noise is a by-product of cheap gradient computation, and does not help training. 

By contrast,
for non-convex optimization problems like training neural networks, noise seems crucial. 
It is observed that with the help of noisy gradients, SGD does not only converge faster, but also converge to a better solution compared with GD \cite{largebatchtraining}. To formally understand this phenomenon, people have analyzed the role of noise in various settings. For example, it is proved that noise helps to escape saddle points \cite{Ge2015,escapesaddlepointsefficiently}, 
gives better generalization \cite{train-faster,generalizationboundsofsgld}, 
and also guarantees polynomial hitting time of good local minima under some assumptions \cite{hittingtime}.

However, it is still unclear why SGD could converge to better local minima than GD. Empirically, in additional to the gradient noise, the step size is observed to be a key factor in optimization. More specifically, small step size helps refine the network and converge to a local minimum, while large step size helps escape the current local minimum and 
go towards a better one  \cite{snapshotensembles,sgdr}. 
Thus, standard training schedule for modern networks uses large step size first, and shrinks it later \cite{resnet,densenet}. 
While using large step sizes to escape local minima
matches with intuition, the existing analysis on SGD for non-convex objectives always considers the small-step-size settings \cite{Ge2015,escapesaddlepointsefficiently,train-faster,hittingtime}.

\begin{figure}
\begin{minipage}{0.48\textwidth}
		\includegraphics[width=\textwidth]{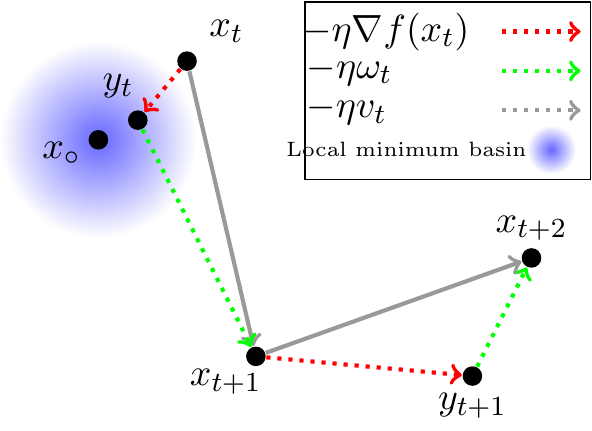}

	\caption{SGD path $x_t\rightarrow x_{t+1}$ can be decomposed into 
	$x_t\rightarrow y_t \rightarrow x_{t+1}$. If the local minimum basin has small diameter, the gradient at $x_{t+1}$ will point away from the basin.
}
	\label{fig:illustration}
\end{minipage}\hspace{10pt}
\begin{minipage}{0.48\textwidth}
	\includegraphics[width=\textwidth]{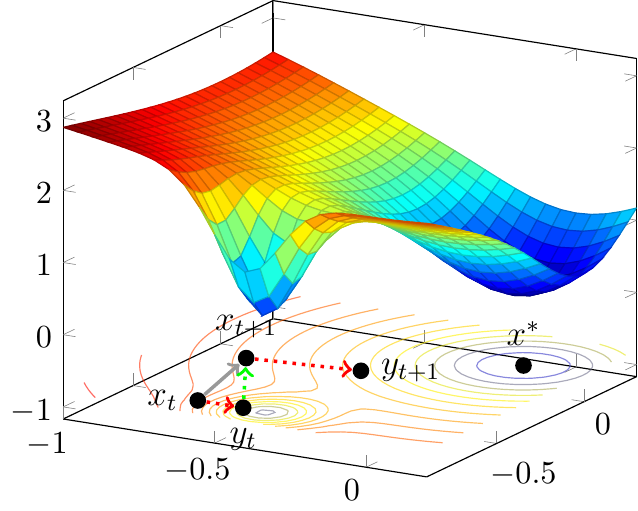}
%
%
%
%
%
%
%
%
%
%
	\caption{3D version of Figure \ref{fig:illustration}: SGD could escape a local minimum within one step. }
	\label{fig:mesh}
\end{minipage}
\end{figure}

See Figure \ref{fig:illustration}  for an illustration. 
Consider the scenario that for some $x_t$, 
instead of pointing to the solution $x^*$ (not shown), its negative gradient points to a bad local minimum $x_{\circ}$, so following the full gradient we will arrive $y_t\triangleq x_t - \eta \nabla f(x_t)$. Fortunately, since we are running SGD, the actual direction we take is $-\eta v_t= -\eta(\nabla f(x_t)+ \omega_t)$, where $\omega_t$ is the noise with $\E[\omega_t]=0, 
\omega_t\sim W(x_t)$\footnote{$W(x_t)$ is data dependent.}. As we show in Figure \ref{fig:illustration}, if we take a large $\eta$, we may get out of the basin region with the help of noise, i.e., from $y_t$ to $x_{t+1}$. Here, getting out of the basin means the negative gradient at $x_{t+1}$ no longer points to $x_{\circ}$ (See also Figure \ref{fig:mesh}). 

%

To formalize this intuition, instead of analyzing the sequence $x_t \rightarrow x_{t+1}$, let us look at the sequence $y_t \rightarrow y_{t+1}$, where $y_t$ is defined to be $x_t - \eta \nabla f(x_t)$, as in the preceding paragraph. The SGD algorithm never computes these vectors $y_t$, but we are only using them as an analysis tool. From the equation $x_{t+1} = y_t - \eta \omega_t$ we obtain the following update rule relating $y_{t+1}$ to $y_t$.
\begin{equation}
y_{t+1}
=y_t- \eta \omega_t
-\eta\nabla f(y_t-\eta\omega_t)
\label{eqn:yt:update}
\end{equation}

The random vector $\eta\omega_t$ in 
(\ref{eqn:yt:update}) has expectation $0$,
so if we take the expectation of both sides of (\ref{eqn:yt:update}),
we get $\E_{\omega_t}[y_{t+1}]=
y_t -\eta\nabla \E_{\omega_t} [   f(y_t- \eta \omega_t)]$.
Therefore, if we define $g_t$ to be the function $g_t(y)
= \E_{\omega_t} [   f(y- \eta \omega_t)]$, which is simply the original function $f$ convolved with the $\eta$-scaled gradient noise, then the sequence $y_t$ is approximately doing gradient descent on the sequence of functions $(g_t)$. 



This alternative view helps to explain why SGD converges to a good local minimum, even when $f$ has many other sharp local minima. 
Intuitively, sharp local minima are eliminated by the convolution operator 
that transforms $f$ to $g_t$, since convolution has the effect of smoothing out short-range fluctuations. 
This reasoning ensures that SGD converges to a good local minimum under much weaker conditions, because instead of imposing convexity or one-point convexity requirements on $f$ itself, we  only require those properties to hold for the smoothed functions obtained from $f$ by convolution. We can formalize the foregoing argument using the following assumption.

\begin{assump}[Main Assumption]
For a fixed point $x^*$\footnote{Notice that $x^*$ is not necessarily the global optimal in the original function $f$ due to the convolution operator.},
noise distribution $W(x)$, step size $\eta$, 
the function $f$ is $c$-one point strongly convex with respect to $x^*$ after convolved with noise. That is,
 for any $x, y$ in domain $\mathbb{D}$ s.t. $y=x-\eta \nabla f(x)$, 
\begin{equation}
\langle -\nabla \E_{\omega \in W(x)}f(y-\eta \omega ), \xs-y\rangle\geq c\|\xs-y\|_2^2
\label{eqn:assumption:conv}
\end{equation}
\label{assump:main:assumption}
\end{assump}

For point $y$, 
since the direction $x^*-y$ points to $x^*$, by having positive inner product with $x^*-y$, we know the direction $-\eta\nabla f(y_t-\eta\omega_t)
$ in (\ref{eqn:yt:update}) approximately points to $x^*$ in expectation (See more discussion on one point convexity in Appendix). Therefore,
$y_t$ will converge to $x^*$ with decent probability:

\begin{thm}[Main Theorem, Informal]
Assume $f$ is smooth, for every $x\in \D$, $W(x)$ s.t., 
$\max_{\omega\sim W(x)}\|\omega\|_2 \leq r$. 
Also assume $\eta$ is bounded by a constant, and
Assumption \ref{assump:main:assumption} holds with $x^*, \eta$, and $c$.
For $T_1\geq \tilde {O}(\frac{1}{\eta c})$\footnote{We use $\tilde {O}$ to hide $\log$ terms here.}, and any $T_2>0$, 
with probability at least $1/2$, 
we have
	$
\|	y_{t}- x^*\|_2^2 \leq O(\log (T_2) \frac{\eta r^2}{c})$
for any $t$ s.t., $T_1+T_2\geq t\geq T_1$. 
\label{thm:main}
\end{thm}

Notice that our main theorem not only says SGD will get close to $\xs$, but also says with constant probability, SGD will stay close to $\xs$ for the future $T_2$ steps. As we will see in Section \ref{sec:empirical}, 
we observe that
Assumption \ref{assump:main:assumption} holds along the SGD trajectory for the modern neural networks when the noise comes from real data mini-batches. Moreover, the SGD trajectory matches with our theory prediction in practice.  

Our main theorem can also help explain 
why SGD could escape ``sharp'' local minima and converge to ``flat'' local minima in practice \cite{largebatchtraining}. Indeed, the sharp local minima have small loss value and small diameter, so after convolved with the noise kernel, 
they easily disappear, which means Assumption \ref{assump:main:assumption} holds. However, flat local minima have large diameter, so they still exists after convolution. In that case, our main theorem says, it is more likely that SGD will converge to flat local minima, instead of sharp local minima.

\begin{figure*}[t]
	\begin{center}
	\includegraphics[width=0.84\textwidth]{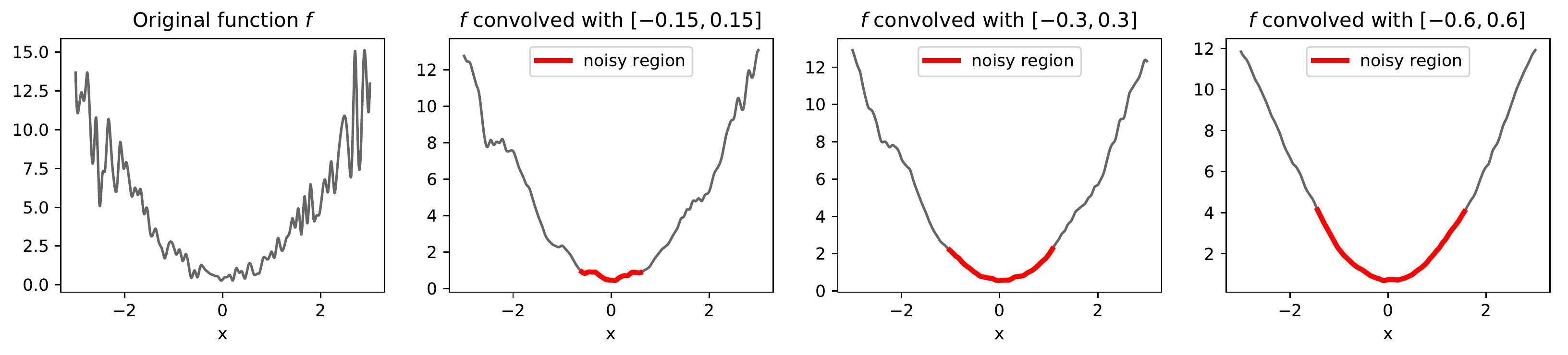}
	\includegraphics[width=0.84\textwidth]{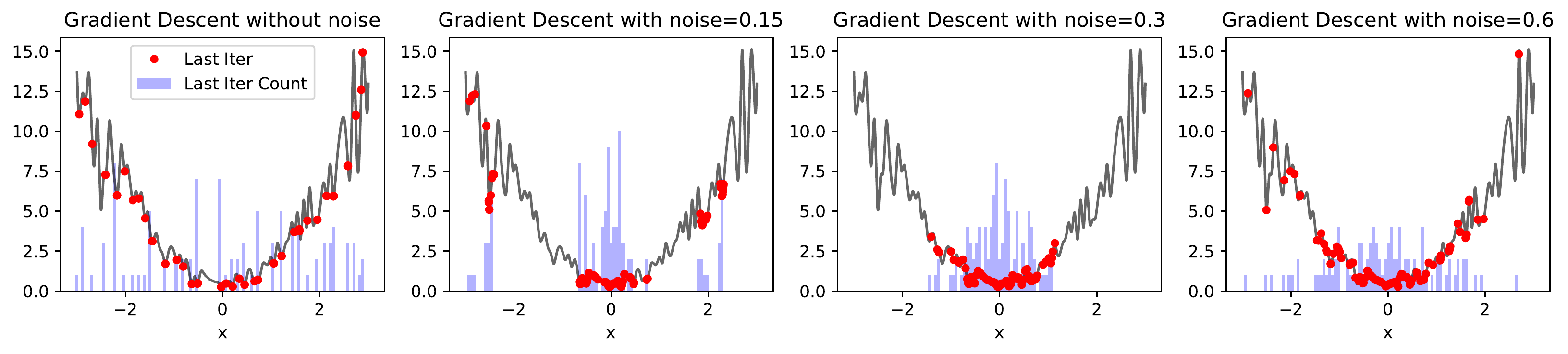}
	\includegraphics[width=0.84\textwidth]{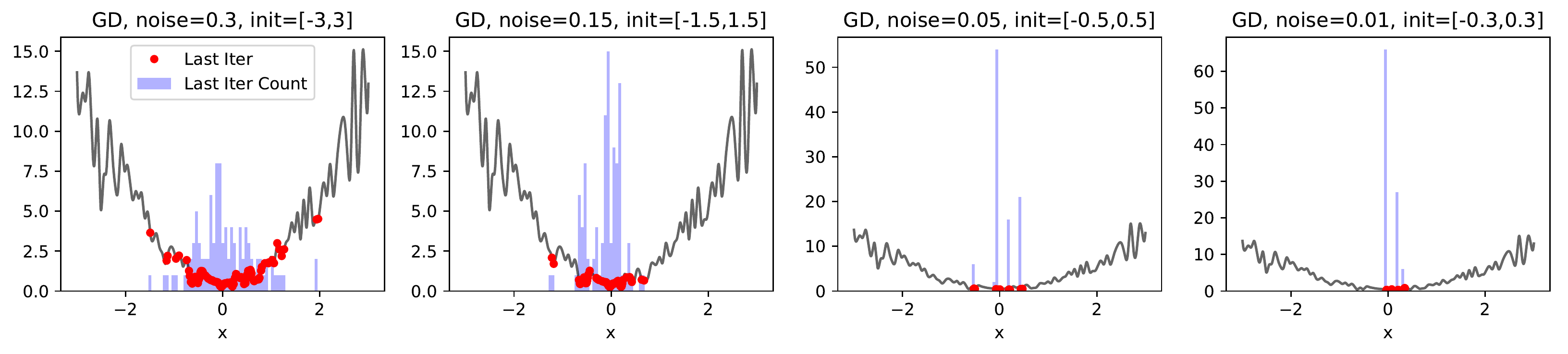}
	\end{center}
	\caption{
		Running SGD on a spiky function $f$. 
		\textbf{Row 1:} $f$ gets smoother after convolving with uniform random noise. 
		\textbf{Row 2:} Run SGD with different noise levels. Every figure is obtained with $100$ trials with different random initializations. Red dots represent the last iterates of these trials, while blue bars represent the cumulative counts. GD without noise easily gets stuck at various local minima, while SGD with appropriate noise level converges to a local region. 
		\textbf{Row 3:} In order to get closer to $x^*$, one may run SGD in multiple stages with shrinking learning rates. 
	}
	\label{fig:gd_noise_converge}	
\end{figure*}

\subsection{Related Work}
Previously, people already realized that the noise in the gradient could help SGD to escape saddle points \cite{Ge2015,escapesaddlepointsefficiently} or achieve better generalization \cite{train-faster,generalizationboundsofsgld}.
With the help of noise, 
SGD can also be viewed as doing approximate Bayesian inference \cite{SGDasapproximatebayesianinference} or 
variational inference \cite{limitcycles}.
Besides, it is proved that SGD with extra noise could ``hit'' a local minimum with small loss value in polynomial time under some assumptions \cite{hittingtime}. However, 
the extra noise is too big to guarantee convergence, and that model cannot deal with escaping sharp local minima. 

Escaping sharp local minima for neural network is important, because it is conjectured (although controversial \cite{sharpminimacangenearlizefordeepnets}) that flat local minima may lead to better generalization
\cite{flat_minima_2,largebatchtraining,entropy-sgd}. It is also observed that the correct learning rate schedule (small or large) is crucial for escaping bad local minima \cite{snapshotensembles,sgdr}.  Furthermore, 
solutions that are farther away from the initialization may lead to wider local minima and better generalization \cite{trainlonger}.
Under a Bayesian perspective, 
it is shown that the noise in stochastic gradient could drive SGD away from sharp minima, which decides the optimal batch size
\cite{abayesianperspectiveongeneralization}. There are also explanations for why small batch methods prefers flat minima while large batch methods are not, by investigating the canonical quadratic sums problem \cite{theimpactoflocalgeomatry}.

To visualize the loss surface of neural network, 
a common practice is projecting it onto a one dimensional line \cite{qualitativelycharacterizing}, which was observed to be convex.  For the simple two layer neural network, a local one point strongly convexity property provably holds under Gaussian input assumption \cite{twolayerconvergence}.


\section{Motivating Example}
Let us first see a  simple example in Figure \ref{fig:gd_noise_converge}. We use $\mathrm{F}_{r,c}$ to denote the sub-figure at row $r$ and column $c$. The function $f$ at $\mathrm{F}_{1,1}$ is a approximately  convex function, but very spiky. Therefore, GD easily gets stuck at various local minima, see $\mathrm{F}_{2,1}$. However, we want to get rid of those spurious local minima, and get a point near $x^*=0$. 

If we take the alternative view
that SGD works on the convolved version of $f$ ($\mathrm{F}_{1,2}$, $\mathrm{F}_{1,3}$, $\mathrm{F}_{1,4}$), we find that those functions are much smoother and contain few local minima. 
However, the gradient noise here is a double-edged sword. On one hand, if the noise is small,
the convolved $f$ is still somewhat non-convex, then
SGD may find a few bad local minima as shown in $\mathrm{F}_{2,2}$. On the other hand, if the noise is too large, the noise dominates the gradient, and SGD will act like random walk, see $\mathrm{F}_{2,4}$.

$\mathrm{F}_{2,3}$ seems like a nice tradeoff, as all trials converges to a local region near $0$, but the region is too big (most points are in $[-1.5,1.5]$). In order to get closer to $0$, we may ``restart'' SGD with a point in $[-1.5,1.5]$, using smaller noise level $0.15$. Recall in $\mathrm{F}_{2,2}$, SGD fails because the convolved $f$ has a few non-convex regions ($\mathrm{F}_{1,2}$), so SGD may find spurious local minima. However, those local minima are outside $[-1.5,1.5]$. The convolved $f$ in $\mathrm{F}_{1,2}$ restricted in $[-1.5,1.5]$ is pretty convex, so if we start a point in this region, SGD converges to a smaller local region centered at $0$, see $\mathrm{F}_{3,2}$.

We may do this iteratively, with even smaller noise levels and smaller initialization regions, 
and finally we will get pretty close to $0$ with decent probability, see $\mathrm{F}_{3,3}$ and $\mathrm{F}_{3,4}$.

\section{Main Theorem}
\label{sec:main:theorem}

\begin{definition}[Smoothness]
	Function $f\in \R^d\rightarrow \R$ is $L$-smooth, if for any $x,y\in \R^d$, 
	\[
	f(y)\leq f(x) + \langle f'(x), y-x \rangle + \frac{L}{2} \|y-x\|_2^2
	\]
\end{definition}
Assume that we are running SGD on the sequence $\{x_t\}$. 
Recall the update rule (\ref{eqn:yt:update}) for $y_t$. Our main theorem says that $\{y_t\}$ is converging to $\xs$ and will stay around $\xs$ afterwards.

\begin{thmn}[\ref{thm:main}]
	Assume $f$ is $L$-smooth, 	for every $x\in \D$, $W(x)$ s.t., 
$\max_{\omega\sim W(x)}\|\omega\|_2 \leq r$.
For a fixed target solution $x^*$, 
if there exists constant $c, \eta> 0$, 
such that 
Assumption \ref{assump:main:assumption} holds with $x^*, \eta, c$, and
$\eta < \min\{\frac1 {2L} , \frac{c}{L^2}, \frac1{2c}\}$, $\lambda \triangleq 2 \eta c -\eta^2 L^2$, 
$b\triangleq \eta^2 r^2 (1+\eta L )^2$. 
Then for any fixed $T_1\geq \frac{\log  (\lambda \|y_0-x^*\|_2^2/b )}{\lambda }$
and $T_2> 0$, 
with probability at least $1/2$, 
we have  
$
\|y_T-x^*\|_2^2 \leq \frac{20b}{\lambda }
$ 
and 
$
\|y_t-x^*\|_2^2 \leq 
O\left (\frac{ \log(T_2)	b}{\lambda }
\right )
$ for all $t$ s.t., $T_1+T_2\geq t\geq T_1$.
\end{thmn}
We defer the proof to Section \ref{appendix:main:thm:proof}.

\textbf{Remark.}
%
For fixed $c$, there exists a lower bound on $\eta$ to satisfy Assumption \ref{assump:main:assumption}, so $\eta$ cannot be arbitrarily small. However, the main theorem says within $T_1+T_2$ steps, SGD will stay in a local region centered at $x^*$ with diameter $O\left (\frac{ \log(T_2)b}{\lambda }
\right )$, which is essentially $\tilde {O}(
\eta r^2/c
)$ that scales with $\eta$. In order to get closer to $x^*$, a common trick in practice is to restart SGD with smaller step size $\eta'$ within the local region. If $f$ inside this region has better geometric properties (which is usually true), one gets better convergence guarantee:

\begin{corollary}[Shrinking Learning Rate]
If the assumptions in Theorem \ref{thm:main} holds, and $f$ restricted in the local region 
$\mathbb{D'}\triangleq \{x| \|x-\xs\|\leq 
\frac{20b}{\lambda}\}$ satisfy the same assumption with $c'>c, \eta'<\eta$, then if we run SGD with $\eta$ for the first $T_1\geq  \frac{\log (\frac{\lambda d}{b})}{\lambda }$ steps, and with $\eta'$ for the next $T_2\geq 
\frac{\log (\frac{\lambda \frac{20b'}{\lambda}}{b'})}{\lambda'}$ steps, 
with probability at least $1/4$, 
we have  
$
\|y_{T_1+T_2}-x^*\|_2^2 \leq \frac{20b'}{\lambda'}
<\frac{20b}{\lambda}
$.
\end{corollary}

This corollary can be easily generalized to shrink the learning rate multiple times. 

Our main theorem is based on the important assumption that the step size is bounded. If the step size is too big, even if the whole function $f$ is one point convex (a stronger assumption than Assumption \ref{assump:main:assumption}), and we run full gradient descent, we may not keep getting closer to $\xs$, as we show below.

\begin{thm}
	\label{thm:large_lr}
For function $f$, if $\forall x, \langle -\nabla f(x), \xs - x\rangle \leq c'\|\xs -x\|_2^2$, and we are at the point $x_t$. If we run full gradient descent with step size $\eta > \frac{2c'\|x_t-\xs\|_2^2 }{\|\nabla f(x_t)\|_2^2}$, we have $\|x_{t+1}-\xs\|_2^2 \geq \|x_t-\xs\|_2^2$. 
\end{thm}
\begin{proof}
	The proof is straightforward and we defer it to Appendix \ref{appendix:large:lr:proof}.
\end{proof}
\begin{wrapfigure}{r}{0.31\textwidth}	
	\vspace{-1em}
	\begin{center}
		\vspace{-1em}
		\begin{tikzpicture}
		
		\node[font= \fontsize{7}{7}\selectfont] (target) at (0,-0.2) {$\xs$};
		\fill (0,0) circle (0.05);	
		\draw[dotted,color=blue!50, thick] (0,0) ++(110:1) arc (110:-10:1);
		\fill (0,0) ++(80:1) circle (0.05);	
		\node[font= \fontsize{7}{7}\selectfont] at (0.2,1.15) {$x_t$};
		\draw[thick,->] (0,0) ++(80:1) --+ (0.4,-0.3);
		\draw[thick,densely dotted, ->] (0,0) ++(80:1) --+ (1,-0.75);
		\node[font= \fontsize{7}{7}\selectfont,text width=36pt] at (2,0.3) {$x_{t+1}$ if\\ $\eta$ too big};
		\end{tikzpicture}
	\end{center}
	\vspace{-1em}
	\caption{When step size is too big, even the gradient is one point convex, we may still go farther away from $\xs$.}
	\vspace{-1em}
	\label{fig:illu_lr_large}
\end{wrapfigure}

This theorem can be best illustrated with Figure \ref{fig:illu_lr_large}. If $\eta$ is too big, although the gradient (the arrow) is pointing to the approximately correct direction, $x_{t+1}$ will be farther away from $\xs$ (going outside of the $\xs$-centered ball). 

Although this theorem analyzes the simple full gradient case, SGD is similar. In the high dimensional case, it is natural to assume that most of the noise will be orthogonal to the direction of $x_t-\xs$, therefore with additional noise inside the stochastic gradient, 
a large step size will drive $x_{t+1}$
away from $\xs$ more easily. 

Therefore, our paper provides a theoretical explanation for why picking step size is so important (too big or too small will not work). We hope it could lead to more practical guidelines in the future.

\section{Proof for Theorem \ref{thm:main}}
\label{appendix:main:thm:proof}

In the proof, we will use the following lemma. 
\begin{thm}[Azuma]
	\label{thm:azuma}
	Let $X_1, X_2, \cdots, X_n$ be independent random variables satisfying 
	$
	|X_i - E(X_i)|\leq c_i, \mathrm{for~} 1\leq i \leq n 
	$. We have the following bound for the sum $X=\sum_{i=1}^n X_i$:
	\[
	\Pr(|X-E(X)| \geq \lambda ) \leq 2 e^{-\frac{\lambda^2}{2\sum_{i=1}^n c_i^2}}.
	\]
\end{thm}


Our proof has four steps. 

\noindent
\textbf{Step 1.} Since Assumption \ref{assump:main:assumption} holds, we show that SGD always makes progress towards $\xs$ in expectation, plus some noise. 

Let filtration $\F_t=\sigma\{
	\omega_0,\cdots, \omega_{t-1}
	\}$, where $\sigma\{\cdot\}$ denotes the sigma field. Notice that for any $\omega_t\sim W(x_t)$, we have $\E[\omega_t | \F_t]=0$.
	
	Thus,
\begin{align*}
&\E[\|y_{t+1}-\xs\|_2^2|\F_t]
=
\E[	\|y_t - \eta \omega_t
-\eta\nabla f(y_t-\eta\omega_t)-\xs\|_2^2
|\F_t]
\\
= &
\E\Big[	\|y_t-\eta\nabla f(y_t-\eta\omega_t)-\xs\|_2^2
+ \|\eta\omega_t\|_2^2
-2 \langle 
\eta\omega_t,
y_t-\eta\nabla f(y_t-\eta\omega_t)-\xs
\rangle 
|\F_t\Big]\\
\leq &
\E\Big[	\|y_t-\eta\nabla f(y_t-\eta\omega_t)-\xs\|_2^2
+ \eta^2 r^2
-2 \langle 
\eta\omega_t,
-\eta\nabla f(y_t-\eta\omega_t)
+\eta \nabla f(y_t)
-\eta \nabla f(y_t)
\rangle 
|\F_t\Big]\\
\leq &
\E\Big[	\|y_t-\xs \|_2^2 + \eta^2\|\nabla f(y_t-\eta\omega_t)\|_2^2
- 
2\eta \langle 
-\nabla f(y_t-\eta\omega_t),
\xs-y_t
\rangle 
+ \eta^2 r ^2 
+ 2 \eta^3 r^2 L 
|\F_t\Big]	\\		
\leq &
\|y_t-\xs \|_2^2 
+
\E\Big[	\eta^2\|\nabla f(y_t-\eta\omega_t)\|_2^2
|\F_t\Big] + \eta^2r^2
- 
2\eta \langle 
-\nabla \E_{\omega_t\in W(x_t)} f(y_t- \eta \omega_t)
,\xs-y_t
\rangle + 2\eta^3 r^2L\\
\leq &
(1-2\eta c)	\|y_t-\xs \|_2^2 + \eta^2 r^2 + 2\eta^3 r^2 L
+
\E\Big[	\eta^2
L^2 \|\xs-y_t+\eta\omega_t\|_2^2
|\F_t\Big]\\
\leq &
(1-2\eta c)	\|y_t-\xs \|_2^2 + \eta^2 r^2 + 2\eta^3 r^2 L
+
\eta^2 L^2 \|\xs-y_t\|_2^2+
\eta^4r^2 L^2\\	
= &
(1-2\eta c+\eta^2 L^2)	\|y_t-\xs \|_2^2 
+ 
\eta^2r^2 (1+\eta L)^2\\	
\end{align*}
\noindent
\textbf{Step 2.} Since SGD makes progress in every step, after many steps, SGD gets very close to $\xs$ in expectation. By Markov inequality, this event holds with large probability. 

Notice that since $\eta< \frac{c}{L^2}$, we have $\lambda =2\eta c - \eta^2L^2> \eta c>0$. 
Recall $b\triangleq \eta^2r^2 (1+\eta L)^2$, we get:
\[
\E[\|y_{t+1}-\xs\|_2^2|\F_t]
\leq (1-\lambda )\|y_t-\xs\|_2^2 + 
b
\] 
Let  $G_t=(1-\lambda )^{-t} (\|y_t-\xs\|_2^2-\frac{b}{\lambda})$, we get:
\[
\E[G_{t+1}|\F_t]\leq G_t
\]
That means, $G_t$ is a supermartingale.
We have 
\[
\E[G_{T_1}|\F_{{T_1}-1}]\leq G_0
\]
Which gives 
\begin{align*}
\E\left [\|y_{T_1}-x^*\|_2^2 -\frac{b}{\lambda} \Big|\F_{{T_1}-1}\right ]&\leq (1-\lambda)^{T_1}
 (\|y_0 - x^*\|_2^2 
-\frac{b}{\lambda}
)\\&\leq 
(1-\lambda)^{T_1} 
\|y_0 - x^*\|_2^2 
\end{align*}
That is, 
\[
\E[\|y_{T_1}-x^*\|_2^2|\F_{{T_1}-1}]\leq 
\frac{b}{\lambda} + (1-\lambda)^{T_1} \|y_0 - x^*\|_2^2 
\]
Since $T_1\geq \frac{\log \left (\frac {\lambda \|y_0 - x^*\|_2^2 }{b}\right )}{\lambda}$, 
we get: 
\[
\E[\|y_{T_1}-x^*\|_2^2|\F_{T_1-1}]\leq 
\frac{2b}{\lambda}
\]
By Markov inequality, we know with probability at least  $0.9$, 
\begin{align}
\|y_{T_1}-x^*\|_2^2 \leq \frac{20b}{\lambda}
\label{eqn:phase_I}
\end{align}

For notational simplicity,
for the analysis below we relabel the point $y_{T_1}$ as $y_0$. Therefore, at time $0$ we already have $\|y_0-x^*\|_2^2\leq \frac{20b}{\lambda}$. 

\noindent
\textbf{Step 3.} Conditioned on the event that we are close to $\xs$, below we show that 
if for $t_0> t\geq 0$, $y_t$ is close to $x^*$,
then $y_{t_0}$ is also close to $\xs$ with high probability.

Let $\zeta=\frac{9T_2}{4}$.
Let event $\EE_t=\{
\forall \tau \leq t, \|y_{\tau}-x^*\|\leq  \mu\sqrt{\frac{b}{\lambda }}
=\delta 
\}$, 
where $\mu$ is a parameter satisfies 
$\mu \geq \max \{
8, 
42
\log^{\frac12} (\zeta)
\}$. If with probability $\frac{5}{9}$,  $\EE_t$ holds for every $t\leq T_2$, we are done. 

By the previous calculation, we know that ($\one_{\EE_t}$ is the indicator function for $\EE_t$)
\[
\E[G_t \one_{\EE_{t-1}} | \F_{t-1}]
\leq G_{t-1}\one_{\EE_{t-1}}
\leq G_{t-1}\one_{\EE_{t-2}}
\]

So $G_t \one_{\EE_{t-1}}$ is a supermartingale, with the initial value $G_0$. In order to apply Azuma inequality, we first bound the following term (notice that we use $\E[\omega_t]=0$ multiple times):

\begin{align*}
&|G_{t+1}\one_{\EE_{t}}-\E[G_{t+1}\one_{\EE_{t}}|\F_t]|\\
=&(1- \lambda)^{-t}|
\|
y_t - \eta \omega_t
-\eta\nabla f(y_t-\eta\omega_t)-x^*
\|_2^2
-
\E[
\|
y_t - \eta \omega_t
-\eta\nabla f(y_t-\eta\omega_t)-x^*
\|_2^2
|\F_t]
|\one_{\EE_{t}}\\
\leq &(1- \lambda)^{-t}|
2\langle 
-\eta \omega_t,
y_t-\eta\nabla f(y_t-\eta\omega_t)-x^*
\rangle +
\|\eta \omega_t\|_2^2
+\|y_t-\eta\nabla f(y_t-\eta\omega_t)-x^*\|_2^2
\\
&-
\E[
2\langle 
-\eta \omega_t,
y_t-\eta\nabla f(y_t-\eta\omega_t)-x^*
\rangle +
\|\eta \omega_t\|_2^2
+\|y_t-\eta\nabla f(y_t-\eta\omega_t)-x^*\|_2^2
|\F_t]
\\
=&(1- \lambda)^{-t}|
\|\eta \omega_t\|_2^2- \E[\|\eta \omega_t\|_2^2|\F_t]
-
2\langle 
\eta \omega_t,
y_t-\eta\nabla f(y_t-\eta\omega_t)-x^*
\rangle +
\|y_t-\eta\nabla f(y_t-\eta\omega_t)-x^*\|_2^2
\\
&-
\E[
2\langle 
\eta \omega_t,
\eta\nabla f(y_t-\eta\omega_t)
\rangle +
\|y_t-\eta\nabla f(y_t-\eta\omega_t)-x^*\|_2^2
|\F_t]
\\
\leq &(1- \lambda)^{-t}|
\eta^2 r^2+
2\eta r \|y_t - x^*\|+
2\langle 
\eta \omega_t,
\eta\nabla f(y_t-\eta\omega_t)
\rangle 
+ \|\eta\nabla f(y_t-\eta\omega_t)-\eta \nabla f(y_t)+
\eta\nabla f(y_t)\|_2^2
\\&-\E[\|\eta\nabla f(y_t-\eta\omega_t)-\eta \nabla f(y_t)+
\eta\nabla f(y_t)\|_2^2|\F_t]
+2\langle y_t-x^*, 
\eta\nabla f(y_t-\eta\omega_t)
\\&-E[\eta\nabla f(y_t-\eta\omega_t)|\F_t]
\rangle 
-
\E[
2\langle 
\eta \omega_t,
\eta\nabla f(y_t-\eta\omega_t)
\rangle 
|\F_t]
\\
\leq &(1- \lambda)^{-t}|
\eta^2 r^2+
2\eta r \|y_t - x^*\|+
4\eta^2 r \|\nabla f(y_t-\eta\omega_t)\|_2
+\eta^2  (
2\eta^2 r ^2 L^2
+ 2\langle 
\nabla f(y_t), \nabla f(y_t- \eta \omega_t)\\&-\nabla f(y_t)
-\E[\nabla f(y_t- \eta \omega_t)-\nabla f(y_t)|\F_t]
\rangle 
 )
+2\eta\Big\langle y_t-x^*, 
\nabla f(y_t-\eta\omega_t)-\nabla f(y_t)
\\&- E[\nabla f(y_t-\eta\omega_t)-\nabla f(y_t)|\F_t]
\Big\rangle 
\\
=&(1- \lambda)^{-t}|
\eta^2 r^2+
2\eta r \|y_t - x^*\|+
4\eta^2 r L ( \eta r + \|y_t-x^*\|_2)
\\&
+\eta^2 \left (
2\eta^2 r ^2 L^2
+ 4
L \|y_t-x^*\|_2 
\eta r L 
\right )
+4\eta^2 rL  \|y_t-x^*\| 
\\
\leq &(1-\lambda)^{-t} \left (
3.5\eta^2 r^2 + 7 \eta r \delta 
\right )
\end{align*}
Where the last inequality uses the fact that $\eta L \leq \frac12$ and $\|y_t-x^*\|_2\leq \delta$ (as $\one_{\EE_{t}}$ holds).
Let $M\triangleq 3.5\eta^2 r^2 + 7 \eta r \delta $.
Let $d_{\tau}=|G_{\tau}\one_{\EE_{\tau-1 }}-\E[G_{\tau }\one_{\EE_{\tau-1}}|\F_t]|$, we have

\[
\sum_{\tau =1}^{t}d_{\tau }^2 = 
\sum_{\tau =1}^{t} (1-\lambda)^{-2\tau } M^2
\]

\[
r_t =\sqrt{\sum_{\tau=1}^t d_{\tau}^2}= M \sqrt{
	\sum_{\tau =1}^{t} (1-\lambda)^{-2\tau }
}
\]

Apply Azuma inequality (Theorem \ref{thm:azuma}), for any $\zeta>0$, we know
\begin{align*}
&\Pr(G_t\one_{\EE_{t-1}} - G_0 \geq \sqrt{2}  r_t \log^{\frac12} (\zeta)) \leq \exp\left (\frac{-2 r_t^2 \log ({\zeta})}{2\sum_{\tau =1}^t d_{\tau}^2} \right )=
\exp^{-\log (\zeta) }= \frac 1 \zeta
\end{align*}

Therefore, 
with probability $1-\frac1{\zeta}$, 
\[
G_t\one_{\EE_{t-1}} \leq G_0 + \sqrt{2}  r_t \log^{\frac12} (\zeta)
\]

\noindent
\textbf{Step 4.} 
The inequality above says, if $\EE_{t-1}$ holds, i.e., for all $\tau\leq t-1, \|y_\tau -\xs \|\leq \delta$, then with probability $1-\frac1{\zeta}$, $G_t$ is bounded. If we can show from the upper bound of $G_t$ that $\|y_t -\xs \|\leq \delta$ is also true, we automatically get $\EE_{t}$ holds. In other words, that means if $\EE_{t-1}$ holds, then $\EE_{t}$ holds with probability $1-\frac1{\zeta}$. Therefore, by applying this claim $T_2$ times, we get $\EE_{T_2}$ holds with probability $1-\frac{T_2}{\zeta}=\frac{5}{9}$.  Combining with inequality (\ref{eqn:phase_I}), we know with probability at least $1/2$, the theorem statement holds. Thus,  it remains to show that 
$\|y_t -\xs \|\leq \delta$.

If $G_t\one_{\EE_{t-1}}  \leq G_0 + \sqrt{2}  r_t \log^{\frac12} (\zeta)$, we know 
\[
(1-\lambda )^{-t} \left (\|y_t - x^*\|_2^2 -\frac{b}{\lambda}\right )
\leq \|y_0 - x^* \|_2^2 -\frac{b}{\lambda}+ \sqrt{2}  r_t \log^{\frac12} (\zeta)
\]
So 
\begin{align*}
&\|y_t - x^*\|_2^2 
\leq (1-\lambda )^t \left ( \|y_0 - x^* \|_2^2 
+
\sqrt{2}  r_t \log^{\frac12} (\zeta)
\right )+\frac{b}{\lambda}
\\\leq &
\|y_0 - x^* \|_2^2 
+\sqrt{2}(1-\lambda )^t 
r_t \log^{\frac12} (\zeta)
+\frac{b}{\lambda}
\end{align*}

Notice that 
\begin{align*}
&(1- \lambda )^t r_t
=
(1- \lambda )^t 
M \sqrt{
	\sum_{\tau =1}^{t} (1-\lambda)^{-2\tau }}
=
M \sqrt{
	\sum_{\tau =1}^{t} (1-\lambda)^{2(t-\tau) }}
\\=&
M \sqrt{
	\sum_{\tau =0}^{t-1} (1-\lambda)^{2\tau }}
\leq 
M 
\sqrt{
	\frac{1}{1-(1-\lambda)^2} 
}\leq  
\frac{M}{\sqrt{\eta c}} 
\end{align*}

The  second last inequality holds because we know 
$
\frac{1}{1-(1-\lambda)^2}=
\frac{1}{2\lambda -\lambda^2}
\leq 
\frac{1}{\lambda}
\leq \frac1{\eta c}
$, since $\lambda= 2\eta c -\eta^2 L^2 \leq 2\eta c <1$, and 
$\lambda > \eta c$. 

That means, 
\begin{align*}
&\|y_t-x^*\|_2^2 
\leq   \|y_0-x^*\|_2^2+   \frac{\sqrt 2 M}{\sqrt{\eta c}} 
\log^{\frac12} (\zeta)+\frac{b}{\lambda}
\\\leq &
\frac{\sqrt 2 (3.5\eta^2 r^2 + 7 \eta r \delta
	)}{\sqrt{\eta c}} 
\log^{\frac12} (\zeta)
+ \frac{21b}{\lambda}
\end{align*}

It remains to prove the following lemma, which we defer to Appendix 
\ref{appendix:small_lemma}.

\begin{lemma}
	\label{lem:small_lemma}
	\[
	\frac{\sqrt 2 (3.5\eta^2 r^2 + 7 \eta r \delta
		)}{\sqrt{\eta c}} 
	\log^{\frac12} (\zeta)
	+ \frac{21b}{\lambda}
	\leq \delta^2
	\]
\end{lemma}
Therefore, $\|y_t -\xs \|\leq \delta$. 
Combining the 4 steps together, we have proved the theorem.

\section{Empirical Observations}
\label{sec:empirical}

\begin{figure*}[t!]
	\centering
	\begin{subfigure}[t]{0.32\textwidth}
		\centering
		\includegraphics[height=1.4in]{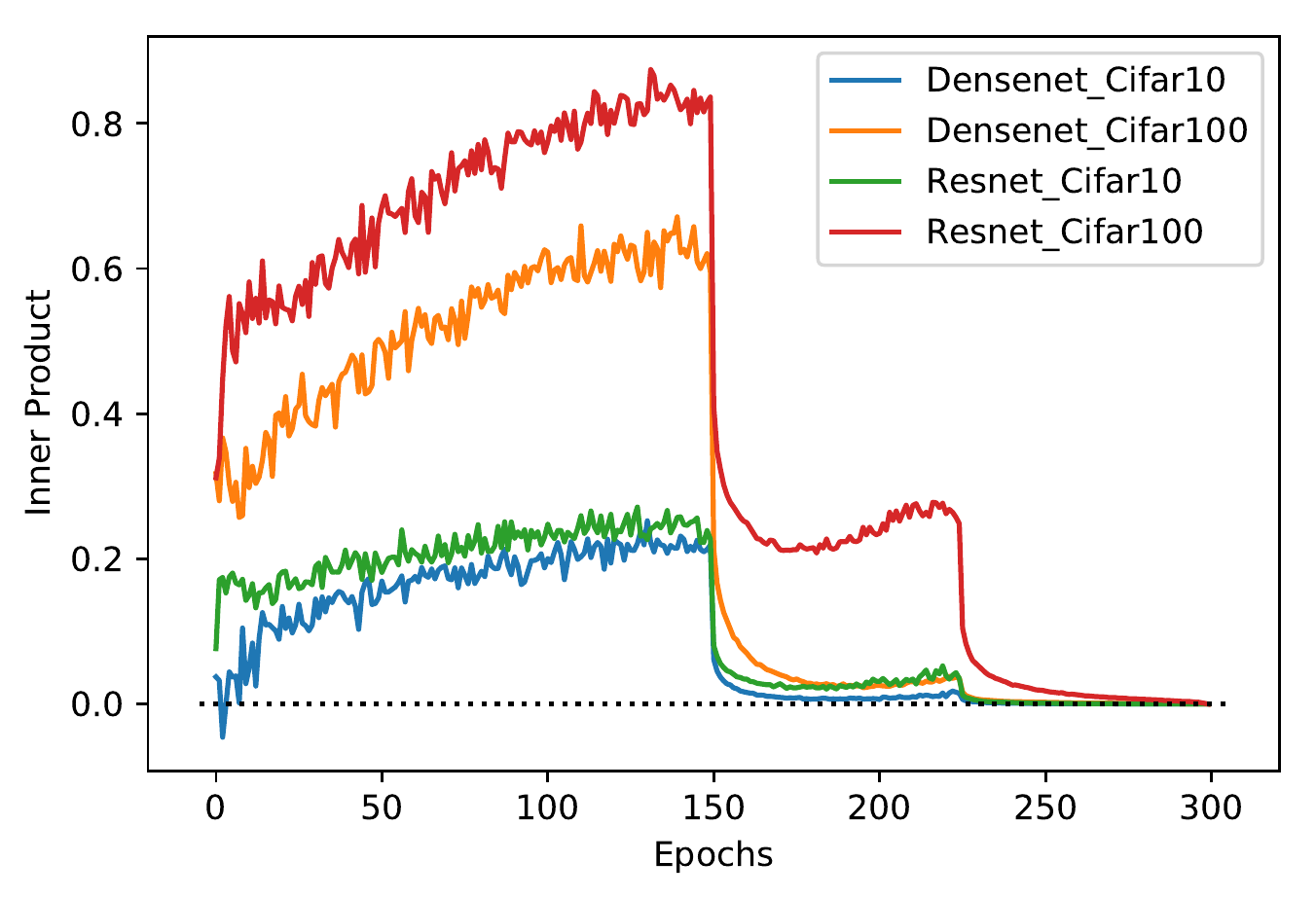}
		\caption{SGD trajectory is locally one point convex.}
		\label{fig:sgd_opc}
	\end{subfigure}%
	~ 
	\begin{subfigure}[t]{0.32\textwidth}
		\centering
		\includegraphics[height=1.4in]{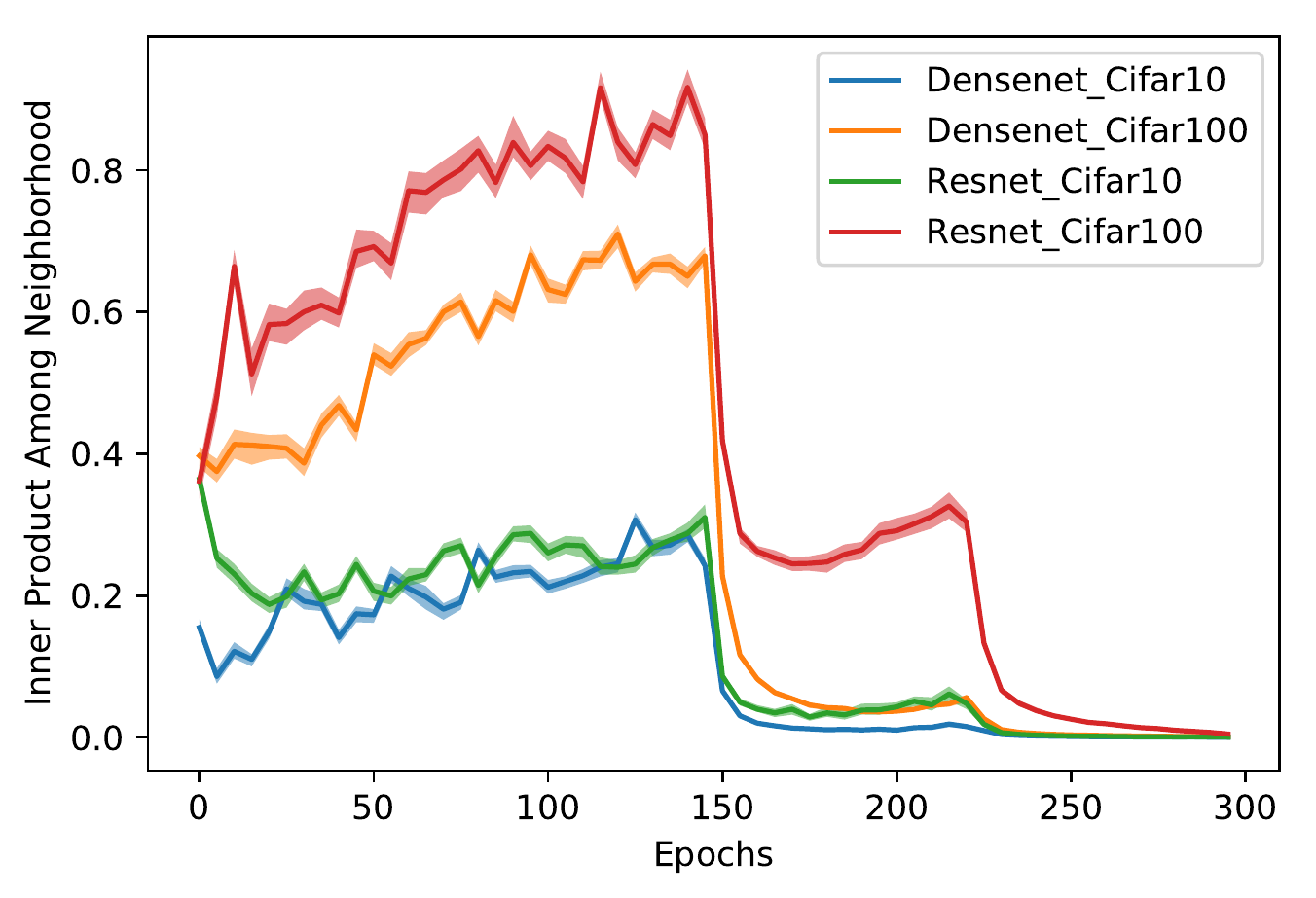}
		\caption{The neighborhood of SGD trajectory is one point convex.}
		\label{fig:neighbor_opc}
	\end{subfigure}
	~
	\begin{subfigure}[t]{0.32\textwidth}
		\centering
		\includegraphics[height=1.4in]{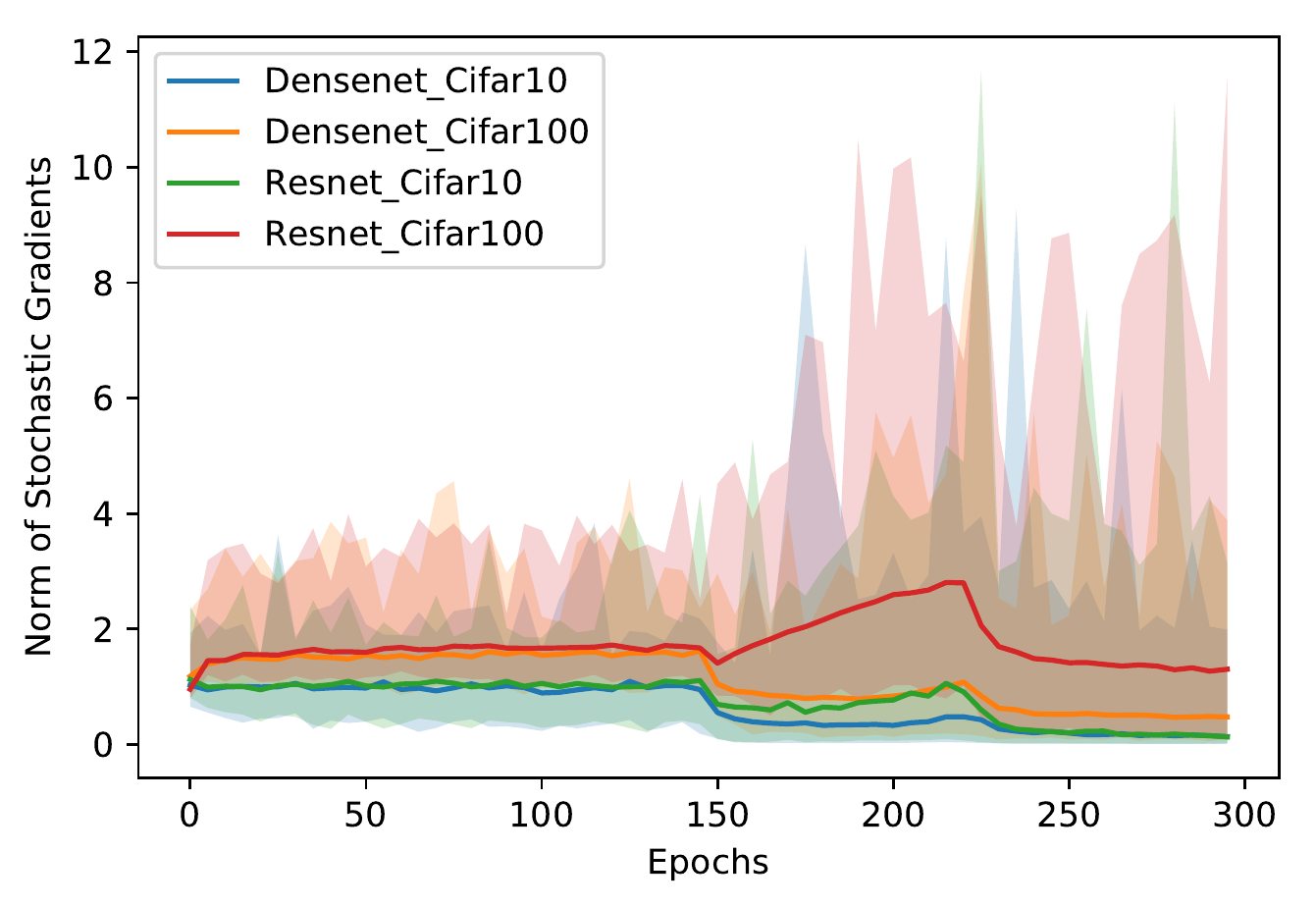}
		\caption{The norm of stochastic gradient}\label{fig:var_sgd}
	\end{subfigure}
	\caption{(a).  The inner product between the negative gradient and $x_{300}-x_t$ for each epoch $t\geq 5$ is always positive. Every data point is the \textbf{minimum} value among $5$ trials. (b). Neighborhood of SGD trajectory is also one point convex with respect to $x_{300}$. (c). Norm of stochastic gradient}
\end{figure*}
In this section, we explore the loss surfaces of modern neural networks, and show that they enjoy many nice one point convex properties. Therefore, our main theorem could be used for explaining why SGD works so well in practice. 


\subsection{The SGD trajectory is one point convex}

It is well known that the loss surface of neural network is highly non-convex, with numerous local minima. However, we observe that the loss surface is consisted of many one point convex basin region, while each time SGD traverses one of such regions. 

See Figure \ref{fig:sgd_opc} for details.
We ran experiments on Resnet \cite{resnet} ($34$ layers, $\approx1.2$M parameters), Densenet \cite{densenet} ($100$ layers, $\approx0.8$M parameters) on cifar10 and cifar100, each for 5 trials with $300$ epochs and different initializations. 
For the start of every epoch $x_t$ in each trial,
we compute the inner product 
between the negative gradient $-\nabla f(x_t)$ and the direction $x_{300}-x_t$. In Figure \ref{fig:sgd_opc}, we plot the minimum value for every epoch among $5$ trials for each setting. Notice that except for the starting period of densenet on Cifar-10, all the other networks in all trials have positive inner products, which shows that the trajectory of SGD (except the starting period) is one point convex with respect to the final solution\footnote{Similar observations were implicitly observed previously \cite{qualitativelycharacterizing}.}.
In these experiments, we have used the standard step size schedule ($0.1$ initially, $0.01$ after epoch $150$, and $0.001$ after epoch $225$). However, 
we got the same observation when using smoothly decreasing step sizes (shrink by $0.99$ per epoch).

\subsection{The neighborhood of the trajectory is one point convex}

Having a one point convex trajectory for $5$ trials does not suffice to show SGD always has a simple and easy trajectory, due to the randomness of the stochastic gradient. Indeed, by a slight random perturbation, SGD might be in a completely different trajectory that is far from being one point convex to the final solution. 
However, in this subsection, we show that it is not the case, as the SGD trajectory is one point convex after convolving with uniform ball with radius $0.5$. That means, the whole neighborhood of the SGD trajectory is one point convex with respect to the final solution. 

In this experiment, we tried  Resnet ($34$ layers, $\approx1.2$M parameters), Densenet ($100$ layers, $\approx0.8$M parameters) on cifar10 and cifar100\footnote{We also tried VGG with $\approx 1$M parameters, but does not have similar observations. This might be why Resnet and Densenet are slightly easier to optimize.}. 
For every epoch in each setting, we take one point and look at its neighborhood with radius $0.5$ (upper bound of the length of one SGD step, as we will show below).
We take $100$ random points inside each neighborhood to verify Assumption \ref{assump:main:assumption}\footnote{We also tried to sample points that are one SGD step away to represent the neighborhood, and got similar observations.}. More specifically, for every random point $w$ in the neighborhood of $x_t$, we computer $\langle -\nabla f(w), x_{300}-x_t\rangle$. 
Figure  \ref{fig:neighbor_opc} shows the mean value (solid line), as well as upper and lower bound of the inner product (shaded area). As we can see, the inner products for all epochs in every setting have small variances, and are always positive. Although we could not verify Assumption \ref{assump:main:assumption} by computing the exact expectation due to limited computational resources, from Figure  \ref{fig:neighbor_opc}
and Hoeffding bound (Lemma \ref{lem:hoeffding}), we conclude that 
Assumption \ref{assump:main:assumption} should hold
with high probability.

\begin{lemma}[Hoeffding bound \cite{hoeffding}]
	\label{lem:hoeffding}
	Let $X_1, \dots, X_n$ be i.i.d.~random variables bounded by the interval $[a, b]$. Then 
	$
	\Pr\left (
	\frac1n \sum_i X_i - 
	\E [X_1]\geq t
	\right )
	\leq  \exp\left (
	- \frac{2nt^2}{(b-a)^2}
	\right )
	$.
\end{lemma}

Figure \ref{fig:var_sgd} shows the norm of the stochastic gradients, including both the mean value (solid lines), as well as upper and lower bounds (shaped area). For all settings, the stochastic gradients are always less than $5$ before epoch $150$ with learning rate $0.1$, and less than $15$ afterwards with learning rate $0.01$. Therefore, 
multiplying step size with gradient norm,
we know SGD step length is always bounded by $0.5$. 

Notice that the gradient norm gets bigger when we get closer to the final solution (after epoch $150$). This further explains why shrinking step size is important. 

\begin{figure*}[t!]
	\centering
	\begin{subfigure}[t]{0.32\textwidth}
		\centering
		\includegraphics[height=1.4in]{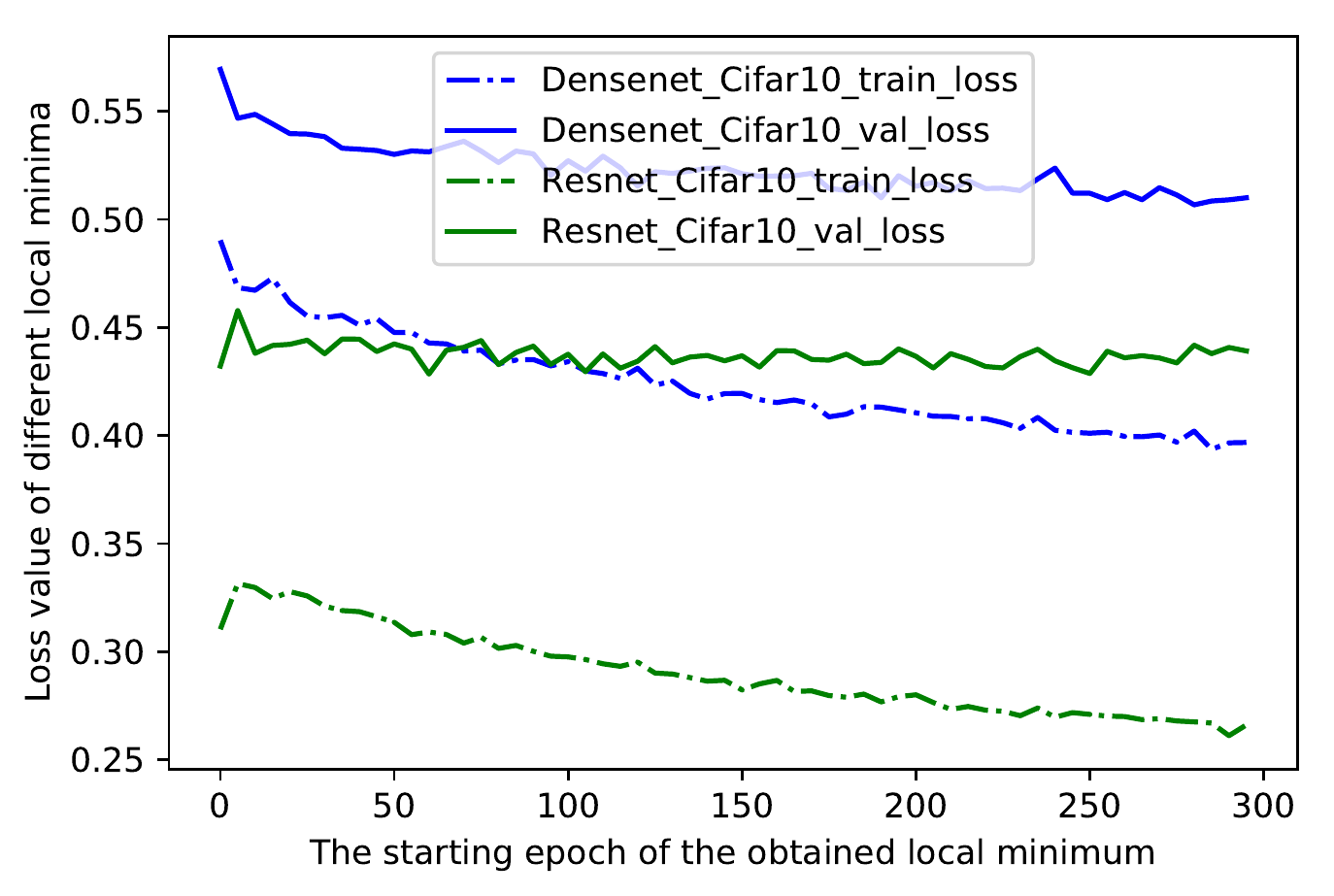}
		\caption{Loss value of different local minima on Cifar10}
		\label{fig:loss_val_1}
	\end{subfigure}%
	~ 
	\begin{subfigure}[t]{0.32\textwidth}
		\centering
		\includegraphics[height=1.4in]{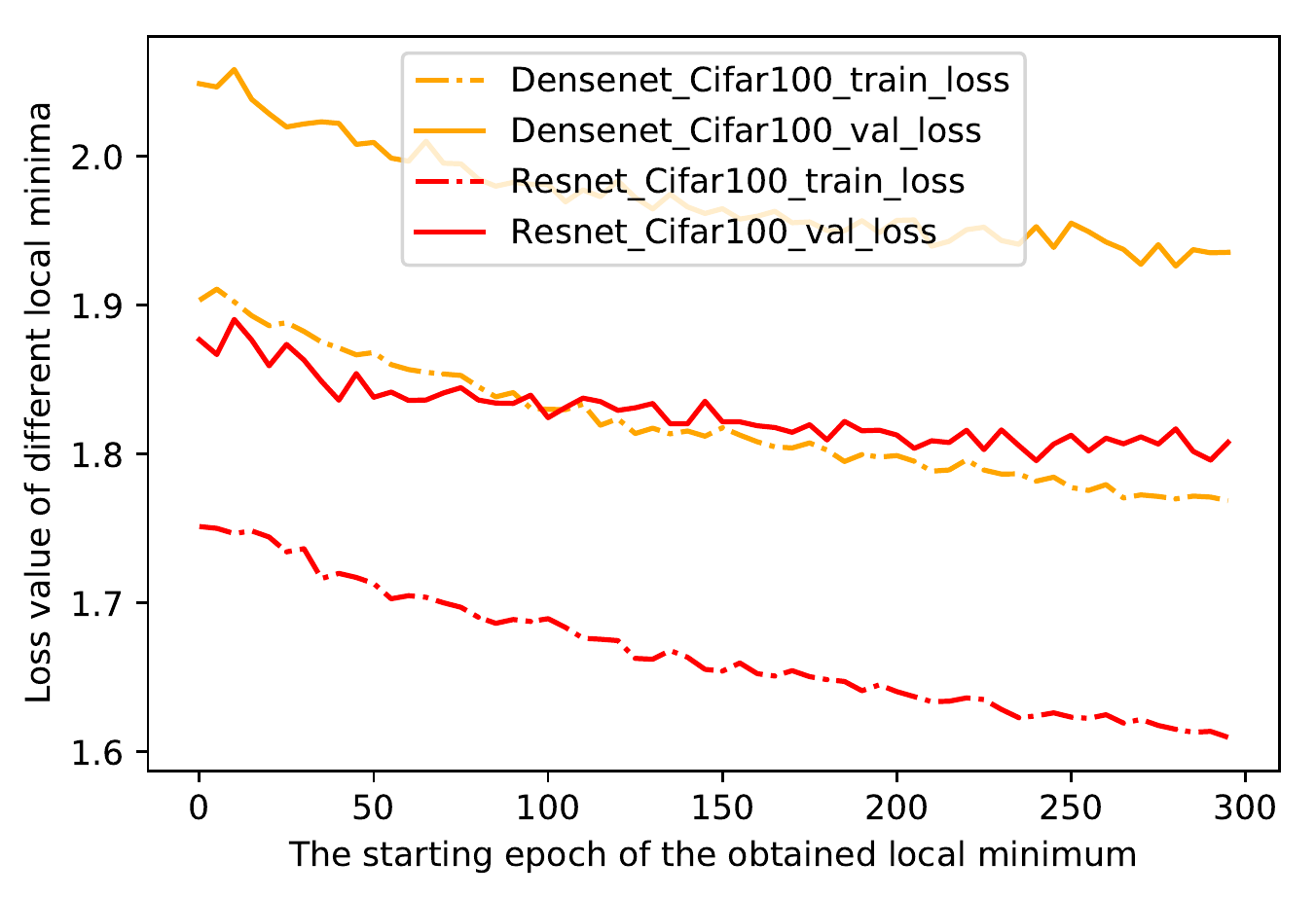}
		\caption{Loss value of different local minima on Cifar100}
		\label{fig:loss_val_2}
	\end{subfigure}
	~
	\begin{subfigure}[t]{0.32\textwidth}
		\centering
		\includegraphics[height=1.4in]{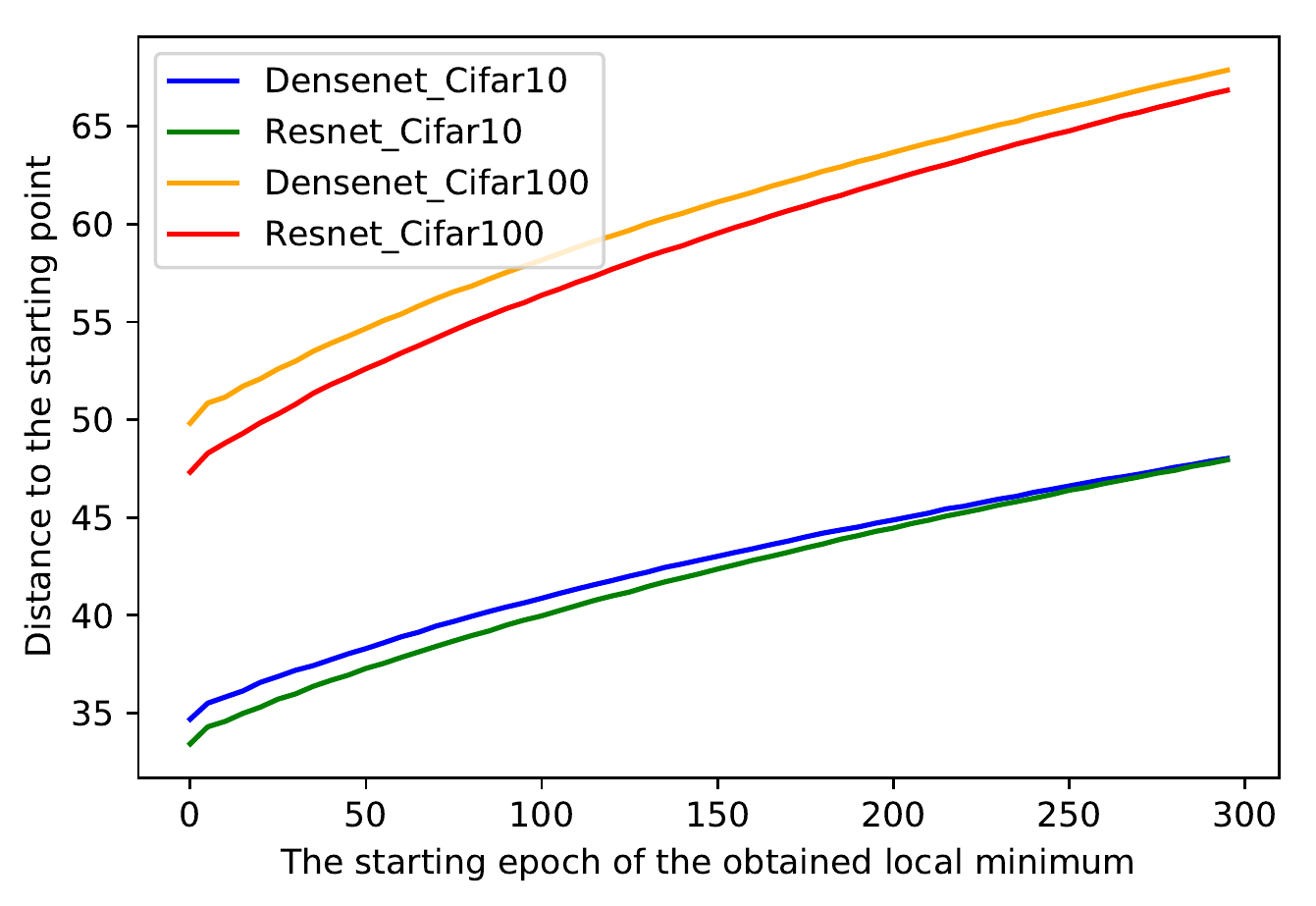}
		\caption{Distance from the local minima to the initialization}
	\end{subfigure}
	\caption{Spectrum of local minima on the loss surface on modern neural networks.}
\end{figure*}
\subsection{Loss surface is locally a ``slope''}

Even with the observation that the whole neighborhood along the SGD trajectory is one point convex with respect to the final solution, 
there exists a chicken-and-egg concern, as the final target is generated using the SGD trajectory. 

In this subsection, we show that the one point convexity is a pretty ``global'' property.
We were running Resnet and Densenet on Cifar10, but with smaller networks (each with about $10K$ parameters). For each network,
if we fix the first $10$ epochs, and generate $50$ SGD trajectories with different random seeds for $140$ epochs and $0.1$ learning rate, we get $50$ different final solutions (they are pretty far away from each other, with minimum pairwise distance $40$). 
For each network, if we look at the inner product between the negative gradient of \textbf{any} epoch of \textbf{any} trajectories, and the vector pointing to \textbf{any} final solutions, we find that the inner products are almost always positive. (only $0.1\%$ of the inner products are not positive for Densenet, and only $2$ out of $343,000$ inner products are not positive for Resnet).  

This indicates that the loss surface is ``skewed'' to the similar direction, and our observation that the whole SGD trajectory is one point convex w.r. to the last point is not a coincidence. Based on our Theorem \ref{thm:main}, such loss surface is very friendly to SGD optimization, even with a few exceptional points that are not one point convex with respect to the final solution.  

Notice that in general, it is not possible that all the negative gradients of all points are one point convex with respect to multiple target points. For example, if we take $1D$ interpolation between any two target points, we could easily find points that have negative gradients only pointing to one target point. However, based on our simulation, empirically SGD almost never traverse those regions.

\subsection{Spectrum of the local minima}

From the previous subsections, we know that the loss surface of neural network has great one point convex properties. It seems that by our Theorem \ref{thm:main}, SGD will almost always converge to a few target points (or regions). However, empirically SGD converges to very different target points. In this subsection, we argue that this is because of the learning rate is too big for SGD to converge (Theorem \ref{thm:large_lr}). On the other hand, whenever we shrink the learning rate to $0.01$, Theorem \ref{thm:main} immediately applies and SGD converges to a local minimum. 

In this experiment, we were running smaller version of Resnet and Densenet (each with about $10K$ parameters) on Cifar10 and Cifar100. For each setting, we first train the network with step size $0.1$ for $300$ epochs, then we pick different epochs as the new starting points for finding nearby local minima using smaller learning rates with additional $150$ epochs. 

See Figure \ref{fig:loss_val_1} and Figure \ref{fig:loss_val_2}. Starting from different epochs, we got local minima with decreasing validation loss and training loss. 

To show that these local minima are not from the same region, we also plot the distance of the local minima to the (unique) initialized point. As we can see, as we pick later epochs as the starting points, we get local minima that are farther away from the initialization with better quality (also observed in \cite{trainlonger}).

Furthermore, 
we observe that 
for every local minimum, 
the whole trajectory is always \textbf{one point convex} to that local minimum. 
Therefore, the time for shrinking learning rate decides the quality of the final local minimum. That is, using large step size initially avoids being trapped into a bad local minimum, and whenever we are distant enough from the initialization, we can shrink the step size and converge to a good local minimum (due to one point convexity by Theorem \ref{thm:main}).

\section{Conclusion}
In this paper, we take an alternative view of SGD that it is working on the convolved version of the loss function. Under this view, we could show that when the convolved function is one point convex with respect to the final solution $\xs$, SGD could escape all the other local minima and stay around $\xs$ with constant probability. 

To show our assumption is reasonable, we look at the loss surface of modern neural networks, and find that SGD trajectory has nice local one point convex properties, therefore the loss surface is very friend to SGD optimization. It remains an interesting open question to prove local one point convex property for deep neural networks.
\section*{Acknowledgement}
The authors want to thank Zhishen Huang for pointing out a mistake in an early version of this paper, and want to thank 
Gao Huang, Kilian Weinberger, Jorge Nocedal,  Ruoyu Sun, Dylan Foster and Aleksander Madry  for helpful discussions. This project is supported by a Microsoft
Azure research award and Amazon AWS research award.

\bibliography{ref}
\bibliographystyle{plain}
\newpage
\appendix
\onecolumn

\section{Discussions on one point convexity}
If $f$ is $\delta$-one point strongly convex around $x^*$ in a convex domain $\mathcal{D}$, then $x^*$ is the only local minimum point in $\mathcal{D}$ (i.e., global minimum). 

To see this, for any fixed $x\in \mathcal{D}$, 
look at the function $g(t) = f( t x^* + (1-t) x )$ for $t \in [0,1]$, then $g'(t) = \langle \nabla f(t x^* + (1-t) x), x^* - x \rangle$. The definition of $\delta$-one point strongly convex implies that the right side is negative for $t\in(0,1]$.  Therefore, $g(t)>g(1)$ for $t>0$.
This implies that for every point $y$ on the line segment joining $x$ to $x^*$, we have $f(y) > f(x^*)$, so $x^*$ is the only local minimum point. 
\section{Proof for Lemma \ref{lem:small_lemma}}
\label{appendix:small_lemma}
\begin{proof}
Recall that we want to show
	\[
	\frac{\sqrt 2 (3.5\eta^2 r^2 + 7 \eta r \delta
		)}{\sqrt{\eta c}} 
	\log^{\frac12} (\zeta)
	+ \frac{21b}{\lambda}
	\leq \delta^2
	=\frac{\mu^2 b}{\lambda }
	=\frac{\mu^2 \eta^2 r^2 (1+\eta L)^2}{\lambda }
	\]
	
On the left hand side there are three summands. Below we show that each of them is bounded by $\frac{\mu^2 b }{3\lambda}$\footnote{We made no effort to optimize the constants here.}. 

Since $\mu \geq \max \{
8, 
42
\log^{\frac12} (\zeta)
\}$, we know $\frac{21 b}{\lambda}\leq 
\frac{63 b}{3\lambda}< \frac{8^2 b}{3\lambda}\leq 
\frac{\mu^2 b }{3\lambda}$.
Next, we have

\begin{align*}
&
42 \log^{\frac12} (\zeta)\leq \mu\\
\Rightarrow
&
\sqrt{30
	\log^{\frac12} (\zeta)
	\eta^{0.5} c^{0.5}}\leq \mu \\
\Rightarrow
&15
\log^{\frac12} (\zeta)
\leq 
\frac{\mu^2   }{2\eta^{0.5} c^{0.5} }\\
\Rightarrow
&\frac{15
}{\sqrt{c}} 
\log^{\frac12} (\zeta)
\leq 
\frac{\mu^2 \eta^{0.5}  }{\lambda }\\
\Rightarrow&\frac{ 3.5\sqrt 2\eta^{1.5} r^2 
}{\sqrt{c}} 
\log^{\frac12} (\zeta)
\leq 
\frac{\mu^2 \eta^2 r^2 }{3\lambda }\\
\Rightarrow&	\frac{ 3.5\sqrt 2\eta^2 r^2 
}{\sqrt{\eta c}} 
\log^{\frac12} (\zeta)
\leq 
\frac{\mu^2 \eta^2 r^2 (1+\eta L)^2}{3\lambda }\\
\end{align*}

Finally, 
\begin{align*}
& 
42
\log^{\frac12} (\zeta)
\leq 
\mu \\
\Rightarrow
& 
\frac{42
}{\sqrt{ c}} 
\log^{\frac12} (\zeta)
\leq 
\mu \sqrt{\frac{ 1}{ c }}
\\
\Rightarrow& 
\frac{ 7\sqrt 2 \eta r 
}{\sqrt{\eta c}} 
\log^{\frac12} (\zeta)
\leq 
\frac {\mu \sqrt{\frac{\eta^2 r^2 (1+\eta L)^2}{2\eta c }}}3
\\
\Rightarrow&\frac{ 7\sqrt 2 \eta r 
}{\sqrt{\eta c}} 
\log^{\frac12} (\zeta)
\leq 
\frac {\delta}3\\
\Rightarrow&
\frac{ 7\sqrt 2 \eta r \delta
}{\sqrt{\eta c}} 
\log^{\frac12} (\zeta)
\leq 
\frac {\delta^2}3
\end{align*}

Adding the three summands together, we get the claim. 
\end{proof}

\section{Proof for Theorem \ref{thm:large_lr}}
\label{appendix:large:lr:proof}
\begin{proof}
	Recall that we have $x_{t+1}=x_{t}- \eta 
	\nabla f(x_t)$. Since we have $\langle -\nabla f(x_t), \xs - x_t\rangle \leq c'\|\xs -x_t\|_2^2$, then 
	\begin{align*}
	\|x_{t+1}-\xs\|_2^2 
	&	=
	\|x_{t}- \eta 
	\nabla f(x_t)-\xs\|_2^2 \\
	&=\|x_t-\xs\|_2^2 +\eta^2\|\nabla f(x_t)\|_2^2 
	- 2\eta \langle \nabla f(x_t), 
	x_t-\xs\rangle \\
	& \geq 
	(1-2\eta c')\|x_t-\xs\|_2^2 +\eta^2\|\nabla f(x_t)\|_2^2 > \|x_t-\xs\|_2^2
	\end{align*}
	Where the last inequality holds since we know $\eta > \frac{2c'\|x_t-\xs\|_2^2 }{\|\nabla f(x_t)\|_2^2}$.
\end{proof}
\end{document}